%% file: colt_arxiv_crc.tex
\documentclass[10pt]{article}
\usepackage{natbib}
\usepackage[colorlinks = true, linkcolor = blue, citecolor = cyan]{hyperref}
\usepackage{amsmath, amsthm, amsfonts, amssymb}
\usepackage{changepage}

\setcitestyle{authoryear,round,citesep={;},aysep={,},yysep={;}}

\renewcommand{\cite}[1]{\citep{#1}}

\usepackage{graphicx,subfigure,color,tikz}

\usepackage{pdfpages, fullpage}

\usepackage{times}
\usepackage[linesnumbered,ruled,vlined]{algorithm2e}

\def\shownotes{1}  
\ifnum\shownotes=1
\newcommand{\authnote}[2]{{$\ll$\textsf{\footnotesize #1 notes: #2}$\gg$}}
\else
\newcommand{\authnote}[2]{}
\fi

\newcommand{\eat}[1]{}

\def\<{\langle}
\def\>{\rangle}

\newcommand{\hide}[1]{}

\newtheorem{thm}{Theorem}
\newtheorem{lem}{Lemma}
\newtheorem*{thms}{Theorem} 

\newtheorem{cor}{Corollary}
\newtheorem{conj}{Conjecture}

\onecolumn

\title{Label optimal regret bounds for online local learning}

\author{Pranjal Awasthi \thanks{Princeton University, Computer Science Department. Email: pawashti@cs.princeton.edu. Supported by NSF grant CCF-1302518.} \and Moses Charikar \thanks{Princeton University, Computer Science Department. Email: moses@cs.princeton.edu. Partially supported by NSF grants CCF-1218687 and CCF-1302518, a Simons Investigator Award, and a Simons Collaboration
Grant.} \and Kevin A. Lai \thanks{Princeton University, Computer Science Department. Email: kalai@cs.princeton.edu} \and Andrej Risteski \thanks{Princeton University, Computer Science Department. Email: risteski@cs.princeton.edu. Partially supported by
NSF grants CCF-0832797, CCF-1117309, CCF-1302518, DMS-1317308, Sanjeev Arora's Simons Investigator Award, and a Simons Collaboration
Grant.}}
\date{\today}

\begin{document}
\maketitle

\begin{abstract}

\input{abstract}

\end{abstract}

\input{intro_conservative_crc}

\section{Computational lower bounds on achievable regret} 

We will proceed with the lower bound first. The overall strategy will be as follows. We will produce an online learning instance from our input graph. In the planted case, there will be a fixed labeling which achieves a large payoff $b_p$, and in the random case, we'll show that any algorithm (efficient or not) can achieve at most some small payoff $b_r$. The reduction will ensure that if we can get a sufficiently low regret $r$ in polynomial time, we will get a payoff of at least $b_p - r$ in the planted case, such that $b_p - r  \gg b_r$, with probability $\frac{4}{5}$. Then to distinguish between planted and random, we simply declare \emph{planted} if the payoff is large enough, and \emph{random} otherwise. 

For both reductions, we will show a ``robust" version of the bound first, e.g. for planted clique, we will show a lower bound of $O(\sqrt{n L^{1-\beta(\epsilon)} T})$ if planted clique is hard when the size of the planted portion is $n^{1/2-\epsilon}$, for some function $\beta(\epsilon)$. Then we will take the limit $\epsilon \to 0$. The details of the reduction follow.   

\subsection{Planted clique-based hardness} 

Let us proceed to the planted clique-based lower bound first. We will show: 

\begin{thm}
Let $\epsilon = \Omega(1)$. If regret $\sqrt{n L^{\beta} T}$ for 
$\displaystyle \beta = \left(1-\omega\left(\frac{1}{\log n}\right)\right)\left(\frac{1}{\frac{1}{2} + \epsilon }\right) - 1 $ 
is achievable in time polynomial in $n, L, T$, then one can distinguish between $G(n,1/2)$ and $G \left( n,1/2,n^{1/2-\epsilon} \right)$ with probability $\frac{4}{5}$\footnote{Again, the choice of $\frac{4}{5}$ is arbitary} in polynomial time. 
\label{t:plantedclique}
\end{thm} 
\begin{proof}

We produce an instance for the online local learning problem, given an instance of the planted clique problem with size of the planted clique $k$ in the following way.  

We randomly partition the input graph into $n' = n/l$ clusters, each containing $l$ vertices, where $l = 10 \frac{n}{k}$. We associate each vertex with a unique label in $\{1,..,l\}$. We then use this as an instance for the online learning problem as follows. We run the online learning game for $T = {n' \choose 2}$ steps. In each step $t$, we query a pair of clusters $(C_{i_t},C_{j_t})$. Each pair is queried once, and the ordering is arbitrary. The algorithm responds with some labeling for the clusters $(l_{i_t}, l_{j_t})$, and the payoff is 1 if the vertex for $l_{i_t}$ in $C_{i_t}$ has an edge to the vertex for $l_{j_t}$ in $C_{j_t}$. Otherwise, the payoff is 0.

The distinguisher for the planted clique problem runs the online learning algorithm on the instance specified above $R = \frac{n^4}{k^{3.7}}$ number of times. This is to ensure that with constant probability, the average payoff of the algorithm over the runs is close to the expected payoff. 
If the average payoff from the $R$ runs is at least $(1+\frac{1}{100})\frac{1}{2}\binom{k/10}{2}$, the distinguisher replies with planted. Otherwise, it replies with random.  

Let's assume the original graph was sampled from $G(n,1/2)$. Then, we claim that any algorithm (regardless if efficient or not) will get an average payoff of at most $\frac{T}{2} + 5 \frac{\sqrt{T}}{2} $ with probability at least $\frac{4}{5}$. 

The above probability is with respect to the randomness in generating the graph from $G(n,1/2)$, the partitioning of the vertices, and any randomness in the algorithm. Let the pair of clusters queried at time step $t$ be $(C_{i_t}, C_{j_t})$. Let's denote the random variable for the payoff in round $t$ on the $r$-th repetition of the online learning problem as $\mathcal{P}^r_{i_t, j_t}$. Let $\mathcal{G}_{a,b}$ be a random 0-1 indicator variable for whether there is an edge between vertices $a,b$. 

If $\mathcal{P}_{i_t,j_t} = \sum_{r=1}^R \mathcal{P}^r_{i_t,j_t}$, then the total payoff of the algorithm is $\mathcal{P} = \sum_{t=1}^T \mathcal{P}_{i_t,j_t}$. 
We claim that the variables $\mathcal{P}_{i_t,j_t}$ are mutually independent. Indeed, this follows because the variables $\mathcal{G}_{a,b}$, for any vertices $a \in C_{i_t}, b \in C_{j_t}$ are independent of the data shown to the online learner in the first $t-1$ rounds and the algorithm's randomness.

But, by linearity of expectation, $\mathbb{E}\left[\frac{1}{R} \mathcal{P}_{i_t, j_t}\right]= \frac{1}{2}$, and $\frac{1}{R} \mathcal{P}_{i_t, j_t}$ always is between 0 and 1. So, by Hoeffding's inequality, 
$$\displaystyle \Pr \left[\frac{1}{R} \sum_{t=1}^T \mathcal{P}_{i_t, j_t}  \geq \frac{T}{2} + 5\frac{\sqrt{T}}{2} \right] \leq e^{-50}$$ 
In particular, with probability at least $\frac{4}{5}$, any algorithm gets average payoff of $\frac{1}{2} \binom{k/10}{2} + o(k^2)$.

Let's proceed to the planted case. First, we claim that with probability at least $\frac{7}{8}$, there is a fixed labeling with payoff at least $\binom{2k/25}{2}$. Let $\mathcal{I}_i$ be an indicator random variable for the event that no vertex from the planted clique belongs to cluster $i$. The partitioning is done independently of the graph, so $\displaystyle \Pr[\mathcal{I}_i = 1] = \left(1 - \frac{10}{k}\right)^{k} \leq e^{-\frac{10}{k}k} = e^{-10}$.  
Hence, if $\mathcal{I}$ is a random variable for the total number of clusters which contain no vertices from the planted clique, we know that 
$\displaystyle \mathbb{E}[\mathcal{I}] = \sum_{i=1}^{n'} \mathbb{E}[\mathcal{I}_i] \leq \frac{k}{10} e^{-10} $. By Markov's inequality,
$\displaystyle \Pr\left[\mathcal{I} \geq \frac{k}{50}\right] \leq \frac{\mathbb{E}[\mathcal{I}]}{\frac{k}{50}} \leq 5 e^{-10} \leq \frac{1}{8} $. 

So, with probability at least $\frac{7}{8}$, the number of clusters with at least one vertex from the planted clique is at least $\frac{k}{10} - \frac{k}{50} = \frac{2k}{25}$. In this case, the labeling where we label each of the clusters with a vertex from the planted clique has a payoff of at least $\binom{2k/25}{2}$.   
In the online learning instance we constructed, the number of vertices is $n'$, the number of rounds is $T$, and the label size is $l$. Let's assume that we can achieve regret of $\sqrt{n'l^{\beta} {T}}$. According to the definition of regret, whenever the graph was a planted instance, and the partitioning resulted in a fixed labeling with payoff at least  $\binom{2k/25}{2}$, the expected payoff of the algorithm (with respect to the randomness of the algorithm) is at least $\displaystyle \binom{2k/25}{2} - \sqrt{n'l^{\beta} T}$. We claim that the average payoff over the $R$ runs of the online learning algorithm will be close to this. 

If we denote by 
$\displaystyle \mathcal{P}^r = \sum_{t=1}^T \mathcal{P}^r_{i_t, j_t}$ the payoff of the algorithm in the $r$-th repetition, then we have that $\mathbb{E}[\mathcal{P}^r] \geq \binom{2k/25}{2} - \sqrt{n'l^{\beta} T}$ and all the variables $\mathcal{P}^r$ are mutually independent and between $0$ and $n^2$.
So, by Hoeffding's bound, $\displaystyle \Pr\left[\frac{1}{R}\sum_{r=1}^R \mathcal{P}^r \leq \mathbb{E}[\mathcal{P}^r] - t\right] \leq e^{-\frac{2R^2 t^2}{R n^4}}$.  
Setting $t = k^{1.9}$, lets us conclude that with probability $1-o(1)$, $\frac{1}{R}\sum_{r=1}^R \mathcal{P}^r \geq \binom{2k/25}{2} - \sqrt{n'l^{\beta} T} - o(k^2)$. Putting everything together, in the planted case, the average payoff is at least $\binom{2k/25}{2} - \sqrt{n'l^{\beta} T} - o(k^2)$ with probability $\frac{4}{5}$. 

Recall that we also proved that in a random instance, we get a payoff at most $\displaystyle \frac{1}{2}\binom{k/10}{2} + o(k^2)$ with probability at least $\frac{4}{5}$. We claim that  
$\displaystyle \binom{2k/25}{2} \geq \left(1+\frac{1}{100}\right) \frac{1}{2} \binom{k/10}{2}$. Indeed, $\binom{2k/25}{2} = \frac{2k/25(2k/25-1)}{2} \geq (1-\frac{1}{100})\frac{(2k/25)^2}{2} $, for large enough $k$, and 
$$ \left(1-\frac{1}{100}\right)\frac{(2k/25)^2}{2} \geq \left(1+\frac{1}{100}\right) \frac{1}{2} \frac{(k/10)^2}{2} \geq \left(1+\frac{1}{100}\right) \frac{1}{2} \binom{k/10}{2}  $$ 

Hence, if 
$\sqrt{n'l^{\beta}T} = o(k^2)$, the distinguisher constructed outputs the correct answer with probability $\frac{4}{5}$. We will show exactly that. 

First we claim that 
\begin{equation}
\label{eq:regret}
l^{\beta} = o \left( \frac{n}{l} \right) 
\end{equation}

Since $l = 10 \frac{n}{k}  = 10 n^{\frac{1}{2} + \epsilon }$,  after rearranging terms, \ref{eq:regret} is equivalent to $\displaystyle n^{\left( \beta + 1 \right) \left( \frac{1}{2} + \epsilon  \right)} = o(n) $.  

Notice that $n^{\omega(\frac{1}{\log n})} = \omega(1)$, so for the above it is sufficient that $\displaystyle \left( \beta + 1 \right) \left( \frac{1}{2} + \epsilon  \right)  = 1 - \omega \left(\frac{1}{\log n}\right) $. 
But since $\displaystyle \beta  = \left(1-\omega\left(\frac{1}{\log n}\right)\right)\left(\frac{1}{\frac{1}{2} + \epsilon }\right) - 1$ the above is clearly satisfied. 

Hence, 
$$\sqrt{n' l^{\beta} T} = \sqrt{\frac{n}{l} l^{\beta} \binom{\frac{n}{l}}{2} } = o\left(\sqrt{l^{\beta} \left(\frac{n}{l}\right)^3}\right) = o\left(\left(\frac{n}{l}\right)^2\right) = o(k^2) $$

which finishes the proof. 

\end{proof}

This quite easily will give the result that assuming Conjecture~\ref{con:clique}, achieving regret $\sqrt{nL^{1-\delta}T}$, for any $\delta = \Omega(1)$ is hard. More precisely: 

\begin{cor} Let $\epsilon = \Omega(1)$. If we can achieve regret $\sqrt{nL^{1-\epsilon}T}$ in time polynomial in $n, L, T$, we can distinguishing between $G(n,1/2)$ and $G(n,1/2,n^{1/2-\frac{\epsilon}{6}}) $ with probability $\frac{4}{5}$ in polynomial time. In particular, if Conjecture~\ref{con:clique} is true, no polynomial time algorithm can achieve regret $\sqrt{nL^{1-\delta}T}$, for any $\delta = \Omega(1)$. 
\label{c:planted} 
\end{cor} 

The proof of this Corollary is straightforward and relegated to Appendix~\ref{a:dense}. We note that a stronger form of Conjecture~\ref{con:clique} is consistent with our current knowledge of planted clique. In particular, we can strengthten the claim to allow any $k = o(\sqrt{n})$, or alternatively $k = n^{\frac{1}{2} - \epsilon}$, for any $\epsilon = \omega(\frac{1}{\log n})$. In this case, Corollary~\ref{c:planted} will imply that achieving regret $\sqrt{n ~ o(L) T}$ is impossible in polynomial time. 

\subsection{Planted dense subgraph hardness} 
\label{s:dense} 

We next move on to the planted dense subgraph based hardness. The proofs in this section are essentially a generalization of the planted clique hardness, so are relegated to Appendix \ref{a:dense}. We formally show: 

\begin{thm} Let $\epsilon, \alpha, k$ satisfy the conditions of Conjecture \ref{con:dense}. If regret $\sqrt{n L^{\beta} T}$ for 
$$ \beta =  2 \frac{\frac{1}{2} - \left( \frac{1}{2} - \epsilon' \right) \left( \alpha + \epsilon \right) - \omega\left(\frac{1}{\log n}\right)}{\frac{1}{2} + \epsilon'}- 1 $$ 
is achievable in time polynomial in $n, L, T$, then one can distinguish between $G(n,p_s)$ and $G(n, p_s, k, p_d)$, where $p_s = n^{-\alpha}, k = n^{\frac{1}{2} - \epsilon'}, p_d = k^{-\alpha - \epsilon}$ with probability $\frac{4}{5}$ in polynomial time.  
\label{t:planteddense}
\end{thm} 

And again as before, assuming Conjecture~\ref{con:dense}, achieving regret $\sqrt{nL^{1-\delta}T}$, for any $\delta = \Omega(1)$ is hard. More precisely:

\begin{cor} 
Let $\epsilon', \alpha, \epsilon = \Omega(1)$ and $\alpha \geq \epsilon$. If we can achieve regret $\sqrt{n L^{1-\epsilon'-\alpha-\epsilon} T}$ in time polynomial in $n, L, T$, we can distinguish between $G(n,p_s)$ and $G(n, p_s, k, p_d)$ in polynomial time with probability $\frac{4}{5}$, where $p_s = n^{-\frac{\alpha}{8}}, k = n^{\frac{1}{2} - \frac{\epsilon'}{4}}, p_d = k^{-\frac{\alpha}{8} - \frac{\epsilon}{8}}$. In particular, if Conjecture~\ref{con:dense} is true, no polynomial time algorithm can achieve regret $\sqrt{nL^{1-\delta}T}$, for any $\delta = \Omega(1)$.     
\label{c:dense}
\end{cor} 

Similarly, a stronger form of Conjecture~\ref{con:dense} is plausible given our current knowledge. We can allow $\alpha = \omega(\frac{1}{\log n})$, $\epsilon = \omega(\frac{1}{\log k})$, and $k = o(\sqrt{n})$. (These constraints are necessary in order to make sure that $n^{-\alpha} = o(1)$, and $k^{-\epsilon} = o(1)$, since unlike planted clique, we are thinking of $p$ and $q$ as asymptotic quantities, so we want to ensure that $k^{-\alpha-\epsilon} = o(k^{-\alpha})$, and $n^{-\alpha} = o(1)$.) In this case, Corollary~\ref{c:dense} will imply that achieving regret $\sqrt{n ~ o(L) T}$ is impossible in polynomial time.

\section{Improved regret bound analysis of log-determinantal regularizer} 

We now move to the other result in our paper: matching the lower bound from the previous section. We show that ``Follow-the-regularized-leader" with the log-determinant-based regularizer from \cite{Christiano} achieves regret $O(\sqrt{n L T})$.

We will follow the \cite{Hazan} framework for online convex optimization. The scenario is as follows: at each round $t$, the player chooses a point $\vec{x}_t \in \mathcal{K}$, where $\mathcal{K}$ is some convex body. A linear payoff function is revealed, and the player receives a payoff $\vec{\mathcal{P}}_t \cdot \vec{x}_t$, for some vector $\vec{\mathcal{P}}_t$. The goal is to compete with the ``best decision in hindsight", i.e. to maximize 
\begin{align*}
\inf_{\vec{\mathcal{P}}_1, \vec{\mathcal{P}}_2, \dots, \vec{\mathcal{P}}_T}  \left\{ \mathbb{E}\left[\sum_{i=1}^T \vec{\mathcal{P}}_i \cdot \vec{x}_i\right] - \max_{\vec{x} \in \mathcal{K}} \sum_{i=1}^T \vec{\mathcal{P}}_i \cdot \vec{x} \right\}
\end{align*}
where the expectation is over the randomness of the algorithm.

Then, ``Follow-the-regularized-leader", with a convex regularizer $\mathcal{R}(\vec{x})$, is the following algorithm:

\begin{algorithm}
 $\vec{x}_1 = \text{argmax}_{\vec{x}} \mathcal{R}(\vec{x})$\; 
 \For{$t \leftarrow 1$ \KwTo $T$} {
 	  Predict $\vec{x}_t$\; 
 	  Observe the payoff function $\vec{\mathcal{P}}_t$\; 
 	  Update $\vec{x}_{t+1} = \text{argmax}_{\vec{x} \in K} \left[ \nu \sum_{s=1}^t \vec{\mathcal{P}}_s \cdot x - \mathcal{R}(\vec{x}) \right]$\; 
 }
 \caption{Follow-the-regularized-leader}
\end{algorithm}

The main theorem in \cite{Hazan} is:

\begin{thms}\cite{Hazan} ``Follow-the-regularized-leader", with a convex regularizer $\mathcal{R}(x)$ and an appropriate choice of $\nu$, achieves regret $O(\sqrt{D\gamma T})$, where 
$$D = \max_{\vec{x} \in \mathcal{K}} |\mathcal{R}(\vec{x})| \; ,\; \gamma = \max_{\vec{x} \in \mathcal{K}, \vec{\mathcal{P}}_t} \vec{\mathcal{P}}_t^\intercal [\nabla^2 \mathcal{R}(\vec{x})]^{-1} \vec{\mathcal{P}}_t $$ 
\end{thms} 

Since we are following the same approach as in \cite{Christiano}, for us the polytope $\mathcal{K}$ will be the convex polytope of pseudo-moments, i.e. positive semidefinite matrices $M_{(i,a), (j,b)}$ where $1 \leq i,j \leq n$, $1 \leq a,b \leq L$, such: 
\begin{itemize}
\item $\forall i,a,j,b$, $1 \geq M_{(i,a),(j,b)} \geq 0$ 
\item $\forall i,j$, $\sum_{a,b} M_{(i,a),(j,b)} = 1$.
\end{itemize} 

Then, $\vec{\mathcal{P}}_t \in [-1,1]^{(nL)^2}$, indexed by all pairs $((i,a), (j,b))$. Furthermore, for any $t$, there are nonzeros in $\vec{\mathcal{P}}_t$ only over a single pair $(i_t,j_t)$ (the edge that round is played on), and in that case $\vec{\mathcal{P}}_t ((i_t,a), (j_t,b))$ is the payoff of playing label $a$ on the $i_t$ vertex, and label $b$ on the $j_t$ vertex. The payoff at round $t$ would be simply  
$$\sum_{a,b} \vec{\mathcal{P}}_t((i_t,a), (j_t, b)) \cdot M_{(i_t,a), (j_t,b)} $$

The regularizer we use is $\mathcal{R}(M) = \log \det(I+ LM)$. In \cite{Christiano}, it is shown that the diameter parameter $D$ is at most $nL$, however an additional $L^2$ factor in the analysis of the $\gamma$ parameter is lost. (While not quite written in these terms there, the argument in the paper can be very easily cast this way.) Here, we improve that analysis to show that in fact $\gamma \leq 4$. 

So, we will simply prove: 

\begin{thm} For online local learning, ``follow-the-regularized-leader" with a regularizer $\mathcal{R}(\vec{x}) = \log \det(I+ LM)$ achieves regret $O(\sqrt{D\gamma T})$, where 
$$D = \max_{\vec{x} \in \mathcal{K}} |\mathcal{R}(\vec{x})| \leq n L \;, \; \gamma = \max_{\vec{x} \in \mathcal{K}, \vec{\mathcal{P}}_t} \vec{\mathcal{P}}_t^\intercal [\nabla^2 \mathcal{R}(\vec{x})]^{-1} \vec{\mathcal{P}_t} \leq 4$$ 
\label{logdet} 
\end{thm}

\subsection{Calculating the inverse Hessian of the regularizer} 

We'll prove the following lemma first: 

\begin{lem} If $\mathcal{R}(M) = \log \det(I+ LM)$, then: 
$$(\nabla^2 \mathcal{R}(M))^{-1}_{((i,a), (j,b)),((i',a'),(j',b'))} = $$
$$ \frac{1}{L^2} \left( \delta\left((i',a'), (j,b)\right) + L \cdot M\left((i',a'), (j,b)\right) \right) \left( \delta \left((i,a), (j',b')\right) + L \cdot M \left( (i,a), (j',b')\right) \right) $$  
\label{inverse}
\end{lem} 
\begin{proof} 

Let's proceed stepwise. First, let's calculate the gradient. For this, the following theorem from matrix calculus is very useful (where adj stands for the adjugate): 

\begin{adjustwidth}{2.5em}{0pt}
\begin{thms} Jacobi's Formula \cite{ magnus1995matrix}: 
$$ \frac{\partial \det(B)}{\partial B_{i,j}} = adj(B)^\intercal_{i,j} = adj(B)_{j,i} = \det(B) B^{-1}_{j,i} $$
\end{thms} 
\end{adjustwidth} 

With this in mind, the gradient is a simple matter of applying the chain rule. To keep the notation clean, let $w = (i,a), x = (j,b)$, and calculate the gradient of $R(M)$ with respect to $M_{w,x}$. We get:   

\begin{align*}
\frac{\partial \mathcal{R}(M)}{\partial M_{w,x}} &= \frac{1}{\det (I + L \cdot M)} \frac{\partial \det(I+L \cdot M)}{\partial M_{w,x}} = (I+L \cdot M)^{-1}_{x,w} \frac{\partial (I+L \cdot M)_{w,x}}{\partial M_{w,x}} \\
&= L (I+L \cdot M)^{-1}_{x,w}
\end{align*}

Again, to keep the notation lighter, let $y = (i',a'), z = (j',b')$. We will use a little bit of matrix calculus to show:  

\begin{adjustwidth}{2.5em}{0pt}

\begin{lem} $\frac{\partial (I+L \cdot M)^{-1}_{x,w}}{\partial M_{y,z}} = - L (I+L \cdot M)^{-1}_{x,y} (I+L \cdot M)^{-1}_ {z,w}$
\end{lem} 
\begin{proof}

Let's denote by $\frac{ \partial X}{ \partial t}$ the matrix with entries $\frac{\partial X_{i,j}}{\partial t}$. Then, we claim the following is true: 
$ \displaystyle \frac{ \partial (X Y)} {  \partial t} = \frac{ \partial X} {\partial t} Y  + X \frac{\partial Y}{\partial t} $. This is not hard to check: it's just due to the fact that in the matrix product XY, the entry $(XY)_{i,j}$ is a sum of terms which multiplications of two entries in X and Y. An application of the chain rule gives the above quite easily. 

Then, we use the following trick: $BB^{-1} = I$, so by the above observation, $\frac{\partial B}{\partial t} B^{-1} + B \frac{\partial B^{-1}}{\partial t} = 0$. 
Hence, $ \frac{\partial B^{-1}}{\partial t} = - B^{-1} \frac{\partial B}{\partial t} B^{-1} $. Let's apply this observation to $B = (I+ L \cdot M)$ and $t = M_{y,z}$ 

\begin{align}
 \frac{\partial (I+L \cdot M)^{-1}_{x,w}}{\partial M_{y,z}} &= - ((I+L \cdot M)^{-1} \frac{\partial (I+L \cdot M)}{\partial M_{y,z}} (I+L \cdot M)^{-1})_{x,w} \\
&= - \sum_{p,q} (I+L \cdot M)^{-1}_{x,p} \frac{\partial (I+L \cdot M)_{p,q}}{\partial M_{y,z}} (I+L \cdot M)^{-1}_{q,w} \label{pq}
\end{align}

Now, the term $\frac{\partial (I+L \cdot M)_{p,q}}{\partial M_{y,z}}$ is non-zero only if $p=y, q=z$, in which case it is equal to $L$. 
Hence, we get:

$$\eqref{pq} = - L (I+L \cdot M)^{-1}_{x,y} (I+L \cdot M)^{-1}_{z,w} $$ 

as needed. 
\end{proof} 
\end{adjustwidth}

With this in mind, the Hessian is obvious: 
\begin{align*}
\frac{\partial^2 \mathcal{R}(M)}{\partial M_{w,x} \partial M_{y,z}} &= \frac{\partial}{\partial M_{y,z}} L(I+L \cdot M)^{-1}_{x,w} = - L^2 (I+L \cdot M)^{-1}_{x,y} (I+L \cdot M)^{-1}_{z,w}
\end{align*}

Let's call the Hessian matrix $H_{(w,x), (y,z)}$. We claim that the inverse $\tilde{H}$ has the following explicit form:
$$\tilde{H}_{(w,x), (y,z)} = - \frac{1}{L^2} (I+L \cdot M)_{x,y} (I+L \cdot M)_{w,z} $$
To show this, it's just a matter of verifying that $(H\tilde{H})_{(w,x),(y,z)} = \delta((w,x),(y,z))$.
\eat{To show this, it's just a matter of verifying that $ \displaystyle \sum_{y,z} H_{(w,x), (y,z)} \tilde{H}_{(y,z),(p,q)} = \delta((w,x), (p,q)) $.}

But this is easy enough: 
\begin{align*}
(H\tilde{H})_{(w,x),(y,z)} &= \sum_{p,q} H_{(w,x), (p,q)} \tilde{H}_{(p,q),(y,z)}\\
&= \sum_{p,q} ( - L^2 (I+L \cdot M)^{-1}_{x,p} (I+L \cdot M)^{-1}_{w,q}) (- 1/L^2 (I+L \cdot M)_{q,y} (I+L \cdot M)_{p,z}) \\
&= \sum_{p} (I+L \cdot M)^{-1}_{x,p}(I+L \cdot M)_{p,z} \sum_{q} (I+L \cdot M)^{-1}_{w,q} (I+L \cdot M)_{q,y} \\
&= \delta(x,z) \delta(w,y) = \delta((w,x), (y,z))
\end{align*}

This finishes the proof of Lemma~\ref{inverse}.

\end{proof}

\subsection{Bounding $\gamma$} 

Finally, we want to estimate $\gamma = \max_{\vec{x} \in \mathcal{K}, \vec{\mathcal{P}}_t} \vec{\mathcal{P}}_t^\intercal [\nabla^2 \mathcal{R}(\vec{x})]^{-1} \vec{\mathcal{P}}_t$, which will be relatively easy.  
Given the form of $\vec{\mathcal{P}}_t$, we can write this as $\displaystyle \sum_{a,b,c,d} \mathcal{P}_{a,b} \mathcal{P}_{c,d} [\nabla^2 \mathcal{R}(\vec{x})]^{-1}_{((i_t, a), (j_t, b)),((i_t,c),(j_t,d))} $
where $(i_t, j_t)$ is the edge chosen at timestep $t$, and $\mathcal{P}_{a,b}$ is the payoff of playing label $a$ on vertex $i_t$ and label $b$ on vertex $j_t$. So, we want to bound 
$$\sum_{a,b,c,d} \mathcal{P}_{a,b} \mathcal{P}_{c,d} [\nabla^2 \mathcal{R}(\vec{x})]^{-1}_{((i_t, a), (j_t, b)), ((i_t, c), (j_t, d))}$$
$$= \sum_{a,b,c,d} -\frac{1}{L^2} \mathcal{P}_{a,b} \mathcal{P}_{c,d} (I+L \cdot M)_{(i_t,c),(j_t,b)} (I+L \cdot M)_{(i_t,a),(j_t,d)} $$

However, since $\mathcal{P}_{a,b}, \mathcal{P}_{c,d} \in [ -1,1 ]$, it suffices to upper bound
$$ \sum_{a,b,c,d} \frac{1}{L^2} (I+L \cdot M)_{(i_t,c),(j_t,b)} (I+L \cdot M)_{(i_t,a),(j_t,d)} = $$ 
$$ \frac{1}{L^2}\sum_{b,c}(I + L \cdot M)_{(i_t,c),(j_t,b)}\sum_{a,d}(I + L \cdot M)_{(i_t,a),(j_t,d)} = \frac{1}{L^2}\left(\sum_{e,f}(I + L \cdot M)_{(i_t,e),(j_t,f)}\right)^2 $$ 


Then we note the following:
\begin{align*}
\sum_{e,f}(I + L \cdot M)_{(i_t,e),(j_t,f)} &= \sum_{e,f} I_{(i_t,e),(j_t,f)} + L \sum_{e,f} M_{(i_t,e),(j_t,f)}\\
&= L + L = 2L
\end{align*}
where we have used the marginalization property of $M$ and the definition of the identity.

\section{Conclusion and open problems} 

In this paper, we studied the optimal regret achievable in polynomial time for online local learning. We showed that follow the regularized leader with a log-determinantal regularizer achieves regret $\sqrt{n L T}$, and we proved a matching lower bound based both on planted clique and planted dense subgraph. 

An interesting open problem is to investigate whether the regret bound can be improved when allowing sub-exponential time algorithms, since both planted clique and planted dense subgraph admit sub-exponential time algorithms. A natural approach is to maintain higher order pseudo-moments, following similar approaches when using the Lasserre/Sum of Squares hierarchies. The key difficulty is the right choice of the regularizer. The log determinant regularizer is one particular approximation of the entropy of a distribution over the set of all possible labelings,
matching the pseudo-moments that we maintain during the algorithm -- it roughly corresponds to the entropy of a Gaussian with matching second moments. \cite{Wainwright} Even if we one has access to higher order moments, it is not clear if there is a better candidate than the log determinant.

Another open problem is basing the hardness of achieving regret $\sqrt{n L T}$ on more standard, worst case assumptions (e.g. NP-hardness, UGC-hardness). Indeed, it isn't obvious that randomness is \emph{required} for proving hardness, but it does seem to help. This mirrors the current state of affairs in improper learning, where the only known hardness results are either based on cryptographic assumptions or very recently, refuting random DNF formulas \cite{daniely2014average}. 

\bibliographystyle{plainnat}
\bibliography{online_local_bibl}

\pagebreak 

\appendix 

\input{apdx_conservative_crc}

\nocite{*} 

\end{document}

%% file: abstract.tex
We resolve an open question from~\cite{christiano2014open} posed in COLT'14 regarding the optimal dependency of the regret achievable for online local learning on the size of the label set. In this framework, the algorithm is shown a pair of items at each step, chosen from a set of $n$ items. The learner then predicts a label for each item, from a label set of size $L$ and receives a real valued payoff. This is a natural framework which captures many interesting scenarios such as online gambling and online max cut. \cite{Christiano} designed an efficient online learning algorithm for this problem achieving a regret of $O(\sqrt{nL^3T})$, where $T$ is the number of rounds. Information theoretically, one can achieve a regret of $O(\sqrt{n \log L T})$. One of the main open questions left in this framework concerns closing the above gap.

In this work, we provide a complete answer to the question above via two main results. We show, via a tighter analysis, that the semi-definite programming based algorithm of~\cite{Christiano} in fact achieves a regret of $O(\sqrt{nLT})$.

Second, we show a matching computational lower bound. Namely, we show that a polynomial time algorithm for online local learning with lower regret would imply a polynomial time algorithm for the planted clique problem which is widely believed to be hard. We prove a similar hardness result under a related conjecture concerning planted dense subgraphs that we put forth. Unlike planted clique, the planted dense subgraph problem does not have any known quasi-polynomial time algorithms.

Computational lower bounds for online learning are relatively rare, and we hope that the ideas developed in this work will lead to lower bounds for other online learning scenarios as well.

%% file: intro_conservative_crc.tex
\section{Introduction} 

Online learning is a classic area of machine learning starting from the seminal work of \cite{littlestone1994weighted}, \cite{desantis1988learning} and \cite{vavock1990aggregating}. In this framework, also known as ``prediction from expert advice'', the learning algorithm has to predict label information about an item or a set of items at each stage. It then earns a real valued payoff which is a function of the predicted labels. The aim is to achieve a total payoff in $T$ rounds comparable to the best expert, i.e., the best fixed labeling of the items. The difference from the best possible payoff is known as the regret of the algorithm.

The weighted majority algorithm \cite{littlestone1994weighted} achieves the optimal regret of $O(\sqrt{T \log N})$ for the above mentioned problem ($T$ is the number of rounds, $N$ is the total number of experts) but is computationally efficient only when the number of experts is small. In many scenarios, one is competing with a set of exponentially many experts. Hence, there has been a significant effort in designing polynomial time algorithms with optimal regret bounds for various such problems such as collaborative filtering, online gambling, and online max cut (\cite{kalai2005efficient}, \cite{hazan2012near}, \cite{kakade2009playing}, \cite{Hazan})

A common aspect of many online learning scenarios mentioned above, is that at each time step, the learner is asked to predict {\em local} information about items. For instance, in the online max cut problem, the learner has to predict whether any two nodes are on the same side of the cut or on opposite sides. Recently, \cite{Christiano} proposed an elegant unifying framework called online local learning to capture such problems. 

In this framework, one is given a set of $n$ items, numbered $1$ to $n$. In each round $t \in [T]$, the learner gets a pair of items $(i_t, j_t)$ as input and has to reply with a pair of labels $(a_{i_t}, b_{j_t})$, where the possible labels are in $[L]$. Then, an adversary picks a payoff function $\mathcal{P}^t: [L]^2 \to [-1,1]$. The goal is to compete with the best \emph{fixed} labeling. More precisely, if we denote 
$$ \text{OPT} = max_{\emph{l} \in [L]^{n}} \sum_{t=1}^T \mathcal{P}^t \left( l(i_t),l(j_t) \right)  $$ 
and the algorithm achieves expected payoff $\text{OPT} - r$, the algorithm has regret $r$, where the expectation is over the algorithm's randomness. 
\eat{
There is a complete graph on $n$ vertices, and the following online learning game is happening at each round $t \in [T]$\footnote{For a number $P \in \mathbb{N}$, $[P]$ denotes the set $\{1,2, \dots, P\}$ }: the learner gets a pair of vertices $(i_t, j_t)$ of the graph as input, and has to reply with a pair of labels $(a_{i_t}, b_{j_t})$, where the possible labels are in $[L]$. Then, an adversary picks a payoff function $C^t: [K]^2 \to [-1,1]$. 

As is common in online learning scenarios, then, the goal is to compete with the best \emph{fixed} labeling. More precisely, if we denote 
$$ OPT = max_{\mathbf{c} \in [K]^{n}} \sum_{t=1}^T C^t \left( c(i_i),c(j_t) \right)  $$ 
we want our algorithm to achieve payoff comparable to OPT in an additive sense. If we can achieve payoff at least $OPT - r$, $r$ is usually called the \emph{regret} of our algorithm. \
}

The main result of \cite{Christiano} is that the well known {\em ``Follow-the-regularized-leader"} algorithm with an appropriate regularizer achieves regret $O(\sqrt{n L^3 T})$ for the online local learning problem. This, in particular, leads to optimal regret bounds\footnote{Up to constant factors} for the online max cut problem. Notice that as mentioned before, one can get the optimal regret of $O(\sqrt{n ~ \log L ~ T})$ via an inefficient algorithm which runs the weighted majority algorithm over the space of all possible labelings. 

One of the main questions left open in this framework was to close the gap between the regret that can be achieved by an efficient algorithm and the information theoretically optimal regret. We close this gap by proving the following results (formal statements appear later).
On the lower bound side, we prove: 
\newtheorem*{t:plantedclique}{Theorem~\ref{t:plantedclique} \textnormal{\textit{(Informal)}}} 
\begin{t:plantedclique}
\label{thm:lower-bound-clique-intro}
For every $\epsilon > 0$, if there exists an algorithm for online local learning achieving regret $O(\sqrt{nL^{1-\beta(\epsilon)}T})$, and running in time polynomial in $n,L,T$, then in polynomial time, one can distinguish an instance of a random graph $G(n,1/2)$ from an instance of $G(n,1/2)$ with a randomly planted clique of size $n^{1/2 - \epsilon}$. Here, $\beta(\epsilon)$ is a function such that $\lim_{\epsilon \rightarrow 0} \beta(\epsilon) = 0$.
\end{t:plantedclique}
We also prove a similar lower bound under a more robust conjecture concerning planting dense subgraphs which we introduce, which has no known quasipolynomial time algorithms, unlike planted clique. We show: 
\newtheorem*{t:planteddense}{Theorem~\ref{t:planteddense} \textnormal{\textit{(Informal)}}}
\begin{t:planteddense}
\label{thm:lower-bound-dense-intro}
For every $\epsilon,\epsilon' > 0$, if there exists an algorithm for online local learning achieving regret $O(\sqrt{nL^{1-\beta(\epsilon, \epsilon')}T})$, and running in time polynomial in $n,L,T$, then in polynomial time, one can distinguish between an instance of $G(n,p)$ and an instance of $G(n,p)$ with a randomly planted instance of $G(k,q)$. Here, $k,q$ depend on $\epsilon,\epsilon'$, and $\beta(\epsilon,\epsilon')$ is a function such that $\lim_{\epsilon,\epsilon' \rightarrow 0} \beta(\epsilon,\epsilon') = 0$. 
\end{t:planteddense}

We match the above lower bounds with the following theorem: 
\newtheorem*{logdet}{Theorem~\ref{logdet} \textnormal{\textit{(Informal)}}}
\begin{logdet}
For the online local learning problem, follow the regularized leader with an appropriate regularizer achieves regret $O(\sqrt{nLT})$.
\label{thm:upper-bound-intro}
\end{logdet}

Jointly these results are meaningful for multiple reasons. First and foremost, online local learning is the most natural generalization of constraint satisfaction problems (CSPs) to the online setting. The semidefinite relaxation upon which Theorem~\ref{logdet} is based is the same one considered in~\cite{raghavendra2008optimal}, who proves that under the Unique Games Conjecture, it actually achieves the best approximation factor among all polynomial-time algorithms. Our result can be viewed as an extension of ~\cite{raghavendra2008optimal}: for the online version of CSPs, follow-the-regularized leader on the same semidefinite relaxation along with a log determinantal regularizer is the ``optimal'' algorithm, under widely believed conjectures. 
Furthermore, while our hardness reduction is specific to the setting of online local learning, given the paucity of  lower bounds in the setting of online learning, our result is a significant contribution and we hope it will find applications in other settings as well. 
Finally, a labeling of the items with $k$ labels can be also viewed as a $k$-partitioning of the items. So, all the above results can be viewed through the lens of online settings for $k$-partitioning. 

\subsection{Techniques}

We obtain the above mentioned upper bound on the regret by showing that ``Follow-the-regularized-leader" using the same regularizer as \cite{Christiano} achieves the regret bound we are claiming, but with a completely different analysis. There, the idea is to express the entropy of a multivariate Gaussian in terms of the log-determinant of its covariance matrix, and that two multivariate Gaussian distributions that differ by a small amount in their covariance matrices cannot be too far in total variation distance as well. The main reason for this approach in \cite{Christiano} is that the Hessian of the log-determinantal regularizer is not diagonal, so it's difficult to argue about its inverse.  We use the special structure of the regularizer to get explicit expressions for the inverse, which allows us then to use more standard tools from convex geometry for analyzing ``Follow-the-regularized-leader". To do this we use some matrix calculus identities, which we think might be useful in other machine learning applications, where one needs to perform regularized optimization over polytopes of pseudo-moments.   

Our lower bounds are based on two conjectures about detecting planted dense structures inside random graphs. The first one is planted clique, which states that detecting planted cliques of sufficiently small size in an Erd\"os-R\'enyi graph cannot be done in polynomial time. We introduce a more robust version of this conjecture, planted dense subgraph, which concerns detecting planted dense Erd\"os-R\'enyi graphs inside sparser ones. While our reductions are similar in both cases, the state of the art algorithms for this detection problem are much worse. This is an indicator that this problem is likely harder and gives even stronger evidence for the hardness of achieving low regret. The proof idea is to use the online learner as an \emph{estimator} of the size of the largest clique or dense subgraph in a graph, and the regret as the \emph{rate} of error in this estimator. We show that if the \emph{rate} is low enough, then one can distinguish between the planted and non-planted case. See next section for further details.

\subsection{The planted dense subgraph and planted clique problems} 

We will review the planted clique conjecture and describe the dense subgraph conjecture, upon which we will be basing our lower bounds.

\subsubsection{Planted clique} 

In the planted clique problem, one is given a graph sampled from one of two possible random ensembles: an Erd\"os-R\'enyi random graph $G(n, 1/2)$, or an Erd\"os-R\'enyi random graph $G(n,1/2)$ along with a clique of size $k$ placed between $k$ randomly chosen vertices in the graph. (The usual notation for this random ensemble is $G(n,1/2,k)$.)  The task is to distinguish whether one is presented with a graph from the $G(n,1/2)$ ensemble or the $G(n,1/2,k)$ ensemble. 

Previous sequences of work \cite{Feldman}, \cite{Potechin}, show that wide classes of natural polynomial time algorithms cannot efficiently distinguish between these two cases when the size of the planted clique is $n^{\frac{1}{2} - \epsilon}$, and it is conjectured that in fact there is no polynomial time algorithm for this task. More precisely, the conjecture is the following: 

\begin{conj} Suppose that an algorithm $\mathcal{A}$ receives as input a graph $G$, which is either sampled from the ensemble $G(n,1/2)$ or $ G(n,1/2, n^{\frac{1}{2} - \epsilon})$, $\epsilon = \Omega(1)$. Then, no $\mathcal{A}$ which runs in polynomial time can decide, with probability $\frac{4}{5}$\footnote{The constant is arbitrary. One could make the conjecture for any constant bounded away from $\frac{1}{2}$}, which ensemble the input was sampled from. 
\label{con:clique}
\end{conj} 
 
\subsubsection{Planted dense subgraph} 

The planted dense subgraph problem is a natural generalization of planted clique, where one again wants to distinguish between a random and a planted instance. In the planted case, we plant a denser graph inside a sparser one. Formally, let $G(n, p, k, q)$ be a random graph ensemble generated in the following manner. First, one picks a random subset $S$ of $k$ vertices. Then, for all pairs of vertices inside $S$, one connects them with an edge independently with probability $q$. For all other pairs of vertices, we connect them independently with probability $p$. 

The sizes and densities of the planted and ambient graph in which we will be interested are $p = n^{-\alpha}, k = n^{\frac{1}{2} - \epsilon'}, q = k^{-\alpha - \epsilon}$, for $\alpha \leq \frac{1}{2}$. The main reason this scenario is interesting is that unlike planted clique, we do not know of quasi-polynomial time algorithms for it. 

To be formal, we conjecture the following: 

\begin{conj} Suppose that an algorithm $\mathcal{A}$ receives as input a graph $G$, which is either sampled from the ensemble $G(n,p)$ or $ G(n,p,k,q)$, where 
$k = n^{\frac{1}{2} - \epsilon'}$ for $\epsilon' = \Omega(1)$  and $k = n^{\Omega(1)}$; $p = n^{-\alpha}$ for $\alpha = \Omega(1), \alpha \leq \frac{1}{2}$; $q = k^{-\alpha-\epsilon}$ for $\epsilon = \Omega(1)$; and $p = o(q)$. 
Then, no $\mathcal{A}$ which runs in polynomial time can decide with probability $\frac{4}{5}$ which ensemble the input was sampled from. 
\label{con:dense} 
\end{conj} 

There are few ways to justify this conjecture. First, the current best known algorithm for this distinguishing problem from \cite{Bhaskara1} runs in time $n^{k^{\Theta(\epsilon)}}$. This bound gives a running time of $2^{n^{\Theta(\epsilon)}}$ since $k$ is polynomial in $n$, which is significantly worse than quasi-polynomial. Second, it's possible to show \cite{Bhaskara1} that spectral methods do not work in this regime. It's also easy to check that simple algorithms like outputting the vertices with highest degree do not work either -- since the variance of the degree in the sparser ambient graph dominates the degrees in the denser planted graph. Finally, similar conjectures to this have already been proposed in various contexts in theoretical computer science. (\cite{arora2010computational}, \cite{applebaum2010public}) 

The fact that state of the art algorithms have a much worse running time for this problem in comparison to planted clique is our motivation for putting forth this conjecture. Namely, our reduction of planted clique/planted dense subgraph to online local learning will produce an online learning instance in which the number of items $n'$, the number of rounds $T$ and the label set size $L$ are all polynomial in the size of the input graph. Furthermore, the time to produce the inputs for the learning algorithm will be polynomial as well. Therefore, if $N = \max(T, L, n')$, and we have an algorithm of running time $f(N)$ for online local learning, we get an algorithm for planted clique/planted dense subgraph of running time $\max\left(f\left(\text{poly}\left(n\right)\right), \text{poly}\left(n\right)\right)$. 

This means for instance, if our algorithm for online local learning has running time $f(N)=N^{o(\log N)}$, our reduction would give an algorithm for planted clique with running time $n^{o(\log n)}$. A similar statement holds in the planted dense subgraph case. If our algorithm for online local learning has running time even $f(N) = 2^{N^{o(1)}}$, the reduction would give an algorithm better than the state of the art for planted dense subgraph.

%% file: apdx_conservative_crc.tex
\section{Relegated proofs} 
\label{a:dense}

\newtheorem*{c:planted}{Corollary~\ref{c:planted}}
\begin{c:planted} Let $\epsilon = \Omega(1)$. If we can achieve regret $\sqrt{nL^{1-\epsilon}T}$ in time polynomial in $n, L, T$, we can distinguishing between $G(n,1/2)$ and $G(n,1/2,n^{1/2-\frac{\epsilon}{6}}) $ with probability $\frac{4}{5}$ in polynomial time. In particular, if Conjecture~\ref{con:clique} is true, no polynomial time algorithm can achieve regret $\sqrt{nL^{1-\delta}T}$, for any $\delta = \Omega(1)$. 
\end{c:planted} 
\begin{proof} 

For ease of notation, let's call $\tilde{\epsilon} := \frac{\epsilon}{6}$. Since $\tilde{\epsilon}= \omega(\frac{1}{\log n})$, directly applying Theorem~\ref{t:plantedclique}, to distinguish between $G(n,1/2)$ and $G(n,1/2,n^{1/2-\tilde{\epsilon}})$, it's sufficient to achieve regret $\sqrt{nL^{\beta}T}$, for $\displaystyle \beta = \left(1- \tilde{\epsilon} \right)\frac{2}{1 + 2 \tilde{\epsilon}} -1 $. Since $\frac{2 }{1 + 2 \tilde{\epsilon}} = 2 - \frac{4 \tilde{\epsilon}}{1+2\tilde{\epsilon}}$,  

$$ \left(1-\tilde{\epsilon} \right)\frac{2}{1 + 2 \tilde{\epsilon}} = \left(1-\tilde{\epsilon} \right)\left(  2 - \frac{4 \tilde{\epsilon}}{1+2\tilde{\epsilon}} \right) =  2 - \left(2 + \frac{4}{1+ 2\tilde{\epsilon}}\right) \tilde{\epsilon} + \frac{2\tilde{\epsilon}^2}{1+2\tilde{\epsilon}} \geq 2 - 6 \tilde{\epsilon}= 2 - \epsilon $$    
Hence, if we can achieve regret $\sqrt{n L^{1 - \epsilon} T} $, 
we can distinguish between $G(n,1/2)$ and\\ $G(n,1/2, n^{\frac{1}{2} - \frac{\epsilon}{6}})$, as we needed.  

\end{proof} 
\newtheorem*{t:planteddense2}{Theorem~\ref{t:planteddense}}
\begin{t:planteddense2}
Let $\epsilon, \alpha, k$ satisfy the conditions of Conjecture \ref{con:dense}. If regret $\sqrt{n L^{\beta} T}$ for 
$$ \beta =  2 \frac{\frac{1}{2} - \left( \frac{1}{2} - \epsilon' \right) \left( \alpha + \epsilon \right) - \omega\left(\frac{1}{\log n}\right)}{\frac{1}{2} + \epsilon'}- 1 $$ 
is achievable in polynomial time, then one can distinguish between $G(n,p_s)$ and $G(n, p_s, k, p_d)$, where $p_s = n^{-\alpha}, k = n^{\frac{1}{2} - \epsilon'}, p_d = k^{-\alpha - \epsilon}$ with probability $\frac{4}{5}$ in polynomial time.     
\end{t:planteddense2}
\begin{proof}

We proceed in the same way as in the proof of Theorem~\ref{t:plantedclique}. Namely, we will produce an instance for the online learning algorithm by partitioning our graph randomly into $n' = \frac{n}{l} $ clusters, each of size $\frac{n}{l}$, where $l = 10 \frac{n}{k}$. As before, we will query all $T$ pairs of clusters, and the payoff will be 1 if there is an edge between the labels supplied by the learner, and 0 otherwise. Finally, we run the distinguisher $R=\frac{n^4}{k^{3.7} p^2_d}$ times, and we output planted if the average payoff from the $R$ runs is at least $\frac{1}{2} \binom{2k/25}{2} \cdot p_d$, and otherwise random. 

As before, we claim that in the case when the graph is $G(n,p_s)$, with probability at least $\frac{4}{5}$, any algorithm will achieve average payoff at most   
$ \displaystyle T \cdot p_s + 10\sqrt{T \cdot p_s} = \binom{k/10}{2} \cdot p_s + 10\sqrt{\binom{k/10}{2} \cdot p_s}$. 

We use the same notation as before: the pair of clusters queried at time step $t$ is $(C_{i_t}, C_{j_t})$, the random variable for the payoff in round $t$ on the $r$-th repetition of the online learning problem is $\mathcal{P}^r_{i_t, j_t}$, and $\mathcal{G}_{a,b}$ is a random 0-1 indicator variable for whether there is an edge between vertices $a,b$.

For the same reasons as before, the variables $\mathcal{P}_{i_t,j_t} = \sum_{r=1}^R \mathcal{P}^r_{i_t,j_t}$ are mutually independent. Furthermore, $\mathbb{E}\left[\frac{1}{R} \mathcal{P}_{i_i,j_t}\right] = p_s$, and $\frac{1}{R} \mathcal{P}_{i_t, j_t}$ always is between 0 and 1. So, by Chernoff,
$\displaystyle \Pr \left[ \sum_{t=1}^T \mathcal{P}_t \geq T \cdot p_s \left(1 + \frac{10}{\sqrt{T p_s}}\right) \right] \leq e^{-100/3}$, i.e. 
$\displaystyle \Pr \left[ \sum_{t=1}^T \mathcal{P}_t \geq T \cdot p_s + 10\sqrt{T p_s} ]\right] \leq e^{-100/3}$. In particular, with probability at least $\frac{4}{5}$, any algorithm gets payoff at most $T \cdot p_s + 10\sqrt{T \cdot p_s}$.

In the planted case, completely the same as in Theorem~\ref{t:plantedclique}, with probability $1-5e^{-10} \geq \frac{14}{15}$, there will be at least $\frac{2k}{25}$ clusters which contain a vertex from the planted graph. 

Conditioned on the above event happening, we claim that any labeling that chooses the vertex from the planted graph in the clusters that contain one achieves a 
 payoff of at least $\binom{2k/25}{2} \cdot p_d - 10 \sqrt{\binom{2k/25}{2} \cdot p_d}$ with probability at least $\frac{15}{16}$. To show this, first notice that conditioned on belonging to two different clusters, the probability of an edge existing between two vertices in the planted graph is a Bernoulli $0-1$ variable, which is 1 with probability $p_d$. This is true since the partitioning is done independently from the graph. But then, the payoff is at least $\binom{2k/25}{2} \cdot p_d - 10\sqrt{\binom{2k/25}{2} \cdot p_d}$ with probability at least $1-e^{-100/3} \geq \frac{7}{8}$ by Chernoff.

Hence, in the planted case, again, with probability at least $\frac{7}{8}$, there is a fixed labeling with payoff at least $\binom{2k/25}{2} \cdot p_d - 10\sqrt{\binom{2k/25}{2} \cdot p_d}$. If the regret is $\sqrt{n'l^{\beta}{T}}$, and such a labeling exists, using a Hoeffding bound as before, with probability at least $1-o(1)$ the average payoff will be at least $\binom{2k/25}{2} \cdot p_d - 10 \sqrt{\binom{2k/25}{2} \cdot p_d} - o(k^2 p_d)$. But since $p_s = o(p_d)$ and $k^2 p_d = \omega(1)$, if the regret is $\sqrt{n'l^{\beta}{T}}$, such that $\sqrt{n'l^{\beta}{T}} = o(k^2 \cdot p_d)$, the distinguisher constructed outputs the correct answer with probability at least $\frac{4}{5}$.  

Since $\frac{n}{l} = \Theta(k)$, it's sufficient to show:  
$$\sqrt{n' l^{\beta} T} = o( (\frac{n}{l})^2 k^{-\alpha-\epsilon}) \Leftrightarrow $$

\begin{equation} 
\label{eq:regretdense}
 l^{(\beta+1)/2} = o(n^{\frac{1}{2} - (\frac{1}{2} - \epsilon')(\alpha' + \epsilon)})
\end{equation}
Plugging in $l = 10 \frac{n}{k} = 10 n^{\frac{1}{2} + \epsilon'} $, \ref{eq:regretdense} is equivalent 
$\displaystyle n^{\left( \frac{\beta+1}{2} \right)  \left( \frac{1}{2} + \epsilon' \right)}  = o(n^{\frac{1}{2} - (\frac{1}{2} - \epsilon')(\alpha' + \epsilon)}) $ 

As before, for this it's sufficient that, 

$$ \left( \frac{\beta+1}{2} \right)  \left( \frac{1}{2} + \epsilon' \right)  = \left( \frac{1}{2} - \left( \frac{1}{2}-\epsilon' \right) \left( \alpha + \epsilon \right) \right) - \omega\left({\frac{1}{\log n}}\right) $$ 

It's easy to check for our choice of $\beta$ that this is satisfied, which finishes the proof. 
\end{proof}

\newtheorem*{c:dense}{Corollary~\ref{c:dense}}
\begin{c:dense}
Let $\epsilon', \alpha, \epsilon = \Omega(1)$ and $\alpha \geq \epsilon$. If we can achieve regret $\sqrt{n L^{1-\epsilon'-\alpha-\epsilon} T}$ in polynomial time, we can distinguish between $G(n,p_s)$ and $G(n, p_s, k, p_d)$ in polynomial time with probability $\frac{4}{5}$, where $p_s = n^{-\frac{\alpha}{8}}, k = n^{\frac{1}{2} - \frac{\epsilon'}{4}}, p_d = k^{-\frac{\alpha}{8} - \frac{\epsilon}{8}}$. In particular, if Conjecture~\ref{con:dense} is true, no polynomial time algorithm can achieve regret $\sqrt{nL^{1-\delta}T}$, for any $\delta = \Omega(1)$. 
\end{c:dense}
\begin{proof}

For notational ease, let $\tilde{\alpha} = \frac{\alpha}{8}$, $\tilde{\epsilon} = \frac{\epsilon}{8}$, $\tilde{\epsilon'} = \frac{\epsilon'}{4}$. 

First, notice that $p_s = o(p_d)$. Indeed, since $p_s = n^{-\tilde{\alpha}}$ and $p_d = k^{-\tilde{\alpha}-\tilde{\epsilon}}$,  
$$p_s = o(p_d) \Leftrightarrow n^{-\tilde{\alpha}} = o(n^{-(\frac{1}{2} - \tilde{\epsilon'})(\tilde{\alpha}+\tilde{\epsilon})}) $$ 
However, since $\alpha \geq \epsilon$, 

$$\tilde{\alpha} \geq \frac{1}{2} (\tilde{\alpha} + \tilde{\epsilon}) = (\frac{1}{2} - \tilde{\epsilon'})(\tilde{\alpha}+\tilde{\epsilon}) +  \tilde{\epsilon'}(\tilde{\alpha}+\tilde{\epsilon})$$ 

Since $\tilde{\epsilon'}, \tilde{\alpha}, \tilde{\epsilon} = \Omega(1)$, clearly this implies $\displaystyle n^{-\tilde{\alpha}} = o(n^{-(\frac{1}{2} - \tilde{\epsilon'})(\tilde{\alpha}+\tilde{\epsilon})})$ 

Since clearly $\tilde{\epsilon}, \tilde{\epsilon'}, \tilde{\alpha} = \omega(\frac{1}{\log n})$, directly applying Theorem~\ref{t:planteddense}, to distinguish between $G(n,p_s)$ and $G(n, p_s, k, p_d)$, where
$k=n^{\frac{1}{2} - \tilde{\epsilon'}}, p_s = n^{-\tilde{\alpha}}$ and $p_d = k^{-\tilde{\alpha} - \tilde{\epsilon}}$, achieving regret $\sqrt{nL^{\beta}T}$ is sufficient, for
$$\beta = 2 \frac{\frac{1}{2} - 2(\frac{1}{2} + \tilde{\epsilon'})(\tilde{\alpha} + \tilde{\epsilon})}{\frac{1}{2} + \tilde{\epsilon'}} - 1 = \frac{1 -  4(\frac{1}{2} + \tilde{\epsilon'})(\tilde{\alpha} + \tilde{\epsilon})}{\frac{1}{2} + \tilde{\epsilon'}} - 1$$
$$=1 - \frac{2\tilde{\epsilon'} + 4(\frac{1}{2}+\tilde{\epsilon'})(\tilde{\alpha} + \tilde{\epsilon})}{\frac{1}{2} + \tilde{\epsilon'}} \geq 
1 - 4\tilde{\epsilon'}  -8(\frac{1}{2}+\tilde{\epsilon'})(\tilde{\alpha} + \tilde{\epsilon}) \geq 1 - 4\tilde{\epsilon'} - 8\tilde{\alpha} - 8\tilde{\epsilon} $$ 
where the next to last inequality holds since $\tilde{\epsilon'} \geq 0$ and the last since $\tilde{\epsilon'} \leq \frac{1}{2}$. 

So, if we can achieve regret 

$$ \sqrt{n L^{1 - 4\tilde{\epsilon'} - 8\tilde{\alpha} - 8\tilde{\epsilon}} T} = \sqrt{n L^{1 - \epsilon' - \alpha - \epsilon}} $$
 
we can distinguish between $G(n,p_s)$ and $G(n, p_s, k, p_d)$, as we needed. 

\end{proof}